\documentclass[a4paper,conference]{IEEEtran}
\ifCLASSINFOpdf
\else
\fi

\usepackage[utf8]{inputenc} % allow utf-8 input
\usepackage[T1]{fontenc}    % use 8-bit T1 fonts
\usepackage[bookmarks=false]{hyperref}

\usepackage{url}            % simple URL typesetting
\usepackage{booktabs}       % professional-quality tables
\usepackage{amsfonts}       % blackboard math symbols
\usepackage{nicefrac}       % compact symbols for 1/2, etc.
\usepackage{microtype}      % microtypography
\usepackage{graphicx}
\usepackage{color}
\usepackage{subfigure}
\usepackage{tabularx}
\usepackage{booktabs,graphicx}
\usepackage{rotating}
\usepackage{paralist}
\usepackage{enumerate}
\usepackage{amsthm}
\usepackage{amsmath}
\usepackage{caption}
\usepackage{float}

\usepackage{balance}

\newcommand*\rot{\rotatebox{90}}

\newtheorem{theorem}{Theorem}

\newcommand\ie{\emph{i.e.}} 
 
\newcommand\etal{\emph{et al.}}

\begin{document}
\title{3D Geometry-Aware Semantic Labeling of Outdoor Street Scenes}

\author{\IEEEauthorblockN{Yiran Zhong}
\IEEEauthorblockA{Research School of Engineering, ANU\\
Data61, CSIRO\\
Canberra, Australia\\
}
\and
\IEEEauthorblockN{Yuchao Dai}
\IEEEauthorblockA{School of Electronics and Information\\
Northwestern Polytechnical University\\
Xi'an, China}
\and
\IEEEauthorblockN{Hongdong Li}
\IEEEauthorblockA{Research School of Engineering, ANU\\
Australian Centre for Robotic Vision\\
Canberra, Australia}}


%Telephone: (800) 555--1212, Fax: (888) 555--1212}


\maketitle

\begin{abstract}
This paper is concerned with the problem of how to better exploit 3D geometric information for dense semantic image labeling. Existing methods often treat the available 3D geometry information (e.g., 3D depth-map) simply as an additional image channel besides the R-G-B color channels, and apply the same technique for RGB image labeling. In this paper, we demonstrate that directly performing 3D convolution in the framework of a residual connected 3D voxel top-down modulation network can lead to superior results. Specifically, we propose a 3D semantic labeling method to label outdoor street scenes whenever a dense depth map is available. Experiments on the ``Synthia'' and ``Cityscape'' datasets show our method outperforms the state-of-the-art methods, suggesting such a simple 3D representation is effective in incorporating 3D geometric information.
\end{abstract}


% \fi
\IEEEpeerreviewmaketitle

\section{Introduction}
Semantic labeling (semantic segmentation) aims to assign class labels (e.g., ``cars'', ``road'', ``building'', ``pedestrian'') to pixels in an image.  It is an important task in computer vision and pattern recognition, which has found wide-range applications in the areas such as autonomous driving \cite{semantic2016}, robot SLAM\cite{Civera2011}, and augmented reality \cite{Matuszka2013}.

Deep Convolutional Neural Networks (CNNs) have gained tremendous success in almost all high-level vision tasks such as image classification, object detection% (AlexNet \cite{Alexnet}, VGG \cite{vgg2014}, ResNet \cite{Resnet})
, as well as semantic labeling \cite{FCN2015}\cite{badrinarayanan2015segnet}\cite{CP2016Deeplab}.  The 2D convolution is defined in the image coordinate, where the filter is applied in the neighborhood defined by image pixel distance.  Deep encoder-decoder (SegNet \cite{badrinarayanan2015segnet}, dilated convolution (DeepLab-LargeFOV \cite{CP2016Deeplab}) have also been proposed under the same framework.  The success of these models mainly lies in their general modeling ability for complex unseen visual scenes.  

To further improve the performance, deeper and wider networks \cite{word.zifeng.2016} have been proposed, which require massive labeled data during training. Even though these models have achieved state-of-the-art performance on various benchmarking datasets, they do not harness the full potentials of available depth/3D clues for semantic segmentation.  Geometric information provides crucial and discriminative semantic cues for color images. Depth maps generally provide complement information to color images, where the 3D structure of the observed scene has been encoded naturally \cite{fusenet2016accv}. Therefore, semantic labeling will benefit from the availability of depth information. For indoor scenes, Hazirbas \etal \cite{fusenet2016accv} proposed a deep auto-encoder network for semantic labeling, where the encoder consists of two branches of networks that simultaneously extract features from color and depth images and fuse depth features into the color feature maps as the network goes deeper. Furthermore, Ma \etal \cite{MVrgbd2017} proposed to leverage the consistencies between multi-view semantic labeling.

However, most existing works have focused on indoor scene labeling where the size of the scene is limited. For outdoor street scene semantic labeling using depth information is difficult due to the following reasons: 1) difficulty in accurate depth acquisition for outdoor scenes; 2) large variation in scene scales; and 3) lacking of outdoor training datasets with dense depth information.

In this paper, we advocate the benefit of using 3D information for outdoor labeling, and propose a simple and efficient way to use the 3D information.  Specifically, we propose a direct way to represent RGB-D image in its natural 3D space, i.e., the way human sense the surrounding 3D world. Given a color image and the associated depth map (from stereo vision or from LIDAR), we transform the color image into 3D voxel space defined by the 3D position of each pixel, which enables subsequent 3D convolution to cater the 3D geometry in extracting semantic feature maps and thus achieves \emph{3D geometry aware semantic labeling}. 

To learn a geometry-aware representation, we propose a light-weight 3D Res-TDM (Residual connected Top-Down Modulation) structure that can squeeze 3D geometric information from depth map and own high resistance to noise and errors. We have performed experiments on the SYNTHIA dataset with ground truth depth map and the Cityscape dataset with computed disparity map. Experimental results demonstrate that our method outperforms the state-of-the-art semantic labeling methods, which indicates the success of our 3D voxel representation in effectively and efficiently encoding 3D geometric information.

The main contributions of the paper can be summarized as:
\begin{compactenum}[1)]
\item A natural and direct 3D representation to encode RGB-D data, thus representing the semantic cues in 3D;
\item 3D convolution to exploit the geometric constraint for semantic labeling, enabling 3D geometry aware semantic labeling;
\item A light weighted 3D res-TDM structure that can squeeze 3D geometric information from depth map and own high resistance to noises and errors.
\end{compactenum}

% There are two main weaknesses: 1. In order to learn the color variances for objects, the network needs to see enough color patterns to establish the color-semantic label connections. 2. In order to memorize these connections, a deeper and/or wider network structure is needed. 3. It will be hard for these networks to classify the object is real or not, i.e., a photograph of a person printed on a bus's body. This is commonly seen in our real life, but existing RGB semantic labeling methods tend to label it as a pedestrian.  

\section{Related work} 
\textbf{Semantic labeling:} Before the era of deep learning, semantic segmentation has been widely formulated as CRF with hand-craft features and low-level vision cues. The breakthrough in deep learning has also been brought to semantic labeling to learn the nonlinear mapping from image to dense labeling in an end-to-end manner. The most noticeable deep convolutional network based semantic labeling method is FCN\cite{FCN2015}, which takes advantage of existing image classification architectures \cite{Alexnet} \cite{vgg2014} \cite{Resnet}. 
However, the decoder phase of FCN is relatively simple that makes it difficult to train. SegNet\cite{badrinarayanan2015segnet} tackles the above weakness by using an auto-encoder structure. Dilated convolution \cite{CP2016Deeplab} has also been introduced to effectively enlarge the field of view of filters to incorporate larger context without increasing the number of parameters or the amount of computation.

\textbf{RGB-D semantic labeling:} Depth information has been used as an important cue to refine semantic labeling in computer vision %\cite{Ren2012} \cite{Silberman2012} 
\cite{Muller2014}. Zhang \etal \cite{Zhang:ECCV2010} designed hand-crafted depth features such as surface normals, height above ground and neighboring smoothness and put them into a classifier. Saurabh \etal \cite{Gupta2014} geocentrically encoded depth into disparity, height and angle as a HHA representation and proposed a 2.5D proposal for object detection and semantic segmentation. Lai \etal \cite{Lai2014} utilizes HMP3D features in an MRF framework to label objects in 3D scenes. More recently, Li \etal \cite{LiRGBDLSTM2016} fused contextual information from RGB and depth channels by stacking convolutional layers with an LSTM layer, which memorizes both short- and long-range spatial dependencies in an image along vertical direction. Another LSTM-F layer has also been used to integrate contexts from different channels and bi-directional propagation is performed to fuse vertical contexts. Hazirbas \etal \cite{fusenet2016accv} proposed a simpler network with auto-encoder style, where two encoders are used to extract features from RGB image and depth image individually and one decoder is applied to decode RGB-D channels. Extracted depth features are fused with RGB in every encoder layer. In these works, color features and depth features are coupled in a human-designed way, which may fail to exploit the strong correlation between color image and depth map.

\textbf{3D convolution:} Volumetric (i.e., spatially 3D) convolution has been successfully used in video analysis (\cite{Kundu2016cvpr}). VoxNet \cite{maturana_iros_2015} and 3D ShapeNet \cite{Zhirong15CVPR} are two pioneer works in applying 3D convolution on voxelized 3D shapes. Very recently, Song \etal \cite{song2016ssc} introduced 3D voxel representation of volumetric occupancy and simultaneously performed scene completion and scene parsing for indoor scenes. However, both works only preserve 3D structure information for object recognition and discard color information in 3D convolution. Moreover, the output resolution of \cite{song2016ssc} is only $36\times 60 \times 60$, which is insufficient for outdoor applications and it requires large labeled data for network training. Multi-view strategy has also been leveraged to exploit 3D geometry information. MVCNN \cite{su15mvcnn} projects 3D point clouds onto different image planes and converts each view image into CNN features. However, this strategy could not be applied to outdoor semantic labeling task straightforwardly due to the difficulty in warping small objects between different views. 

%Since semantic segmentation is a pixel-wise labeling task on the reference frame, a warping process is needed to project other views back to the reference frame and fuse the results. Due to complex occlusions, some information, especially small objects will lost during the warping process.

By contrast to the above works, we propose to make use of color information as well as 3D structural information for dense semantic labeling under an unified framework. Our light-weight network architecture also allows us to increase the output dimension with a reasonable scale and can be trained from scratch with only thousands of samples.  

%==================================
\section{Our Approach}
Here, we describe our geometry-aware semantic labeling framework by performing 3D convolution in the framework of 3D voxel convolutional neural network. First, by contrast to existing methods that simply treating depth map as an additional channel besides the R-G-B channels, we represent the input RGB-D images in 3D voxel representation, where each voxel is associated with color. Then a top-down module is proposed to exploit the rich 3D geometric information for outdoor scene semantic labeling, where 3D convolution is performed to extract 3D geometry aware features.

\subsection{3D Representation}
Given RGB-D images, existing methods either represent the generic 3D point clouds with volumetric or multi-view representation. The volumetric representation encodes a 3D shape as a 3D tensor of binary or real values while the multi-view representation encodes a 3D shape as a collection of renderings from multiple viewpoints. However these representations are mainly designed for indoor applications and cannot cope with outdoor scenarios for the following reasons : 1) difficulty in accurate depth acquisition; 2) large variation in scene scales; and 3) lack of outdoor training dataset with dense depth information. Furthermore, for a typical driving scene, the depth ranges from 0.5 meters to infinity (i.e., the sky), which makes it impossible to discretize depth values into a certain range. Therefore direct voxelizing in 3D space for outdoor street scene is infeasible.

To cater the above difficulties, we propose a new and yet direct 3D voxel representation for outdoor street scenes. Specifically, instead of resorting to the $XYZ$ space for 3D point clouds, we propose to combine the image coordinate and the disparity directly, thus $UVD$ space, where $(U,V)$ index the 2D image coordinate while $D$ indexes the discrete disparity. At a first glance, this 3D voxel representation may introduce severe distortion in 3D representation. Here we demonstrate that while providing simplicity in representation, the $UVD$ 3D voxel representation also owns much desired geometric property as in the original $XYZ$ space.
\begin{theorem}
Any order curve in the $XYZ$ 3D space corresponds to a 3D curve with the same order in the $UVD$ 3D space.
\end{theorem}
\begin{proof}
Without loss of generality, we take the second order surface in 3D as an example. A second order surface in the $XYZ$ 3D space is defined by the following equation:
\begin{equation}
[X_i, Y_i, Z_i, 1] \mathbf{A} [X_i, Y_i, Z_i, 1]^T = 0,
\end{equation}
where $(X_i,Y_i,Z_i)$ defines the 3D points on the surface and $\mathbf{A} \in \mathbb{R}^{4\times 4}$ indexes the 3D surface. The $UVD$ space and the $XYZ$ space are connected via perspective projection:
\begin{equation}
% \begin{split}
% &X_i = \frac{(u_i-u_0)}{f_x} Z_i = \frac{(u_i-u_0)}{f_x} \frac{fb}{d_i}, \\ &Y_i = \frac{(v_i-v_0)}{f_y} Z_i = \frac{(v_i-v_0)}{f_y} \frac{fb}{d_i}, Z_i = \frac{fb}{d_i},
% \end{split}
X_i = \frac{(u_i-u_0)}{f_x} Z_i, Y_i = \frac{(v_i-v_0)}{f_y} Z_i, Z_i = \frac{fb}{d_i},
\end{equation}
where $f_x, f_y, u_0, v_0$ are the intrinsic parameters of the camera while $f,b$ define the transformation from disparity $d_i$ to 3D coordinate $Z_i$. By substituting these relations into the 3D surface and re-organizing the equation, we have
% \begin{equation}
% \left[
% \begin{array}{cccc}
% \frac{(u_i-u_0)}{f_x} \frac{fb}{d_i}, &
% \frac{(v_i-v_0)}{f_y} \frac{fb}{d_i}, &
% \frac{fb}{d_i}, &
% 1
% \end{array}
% \right] \mathbf{A} \left[
% \begin{array}{cccc}
% \frac{(u_i-u_0)}{f_x} \frac{fb}{d_i}, &
% \frac{(v_i-v_0)}{f_y} \frac{fb}{d_i}, &
% \frac{fb}{d_i}, &
% 1
% \end{array}
% \right]^T = 0.
% \end{equation}
% By re-organizing the equation, we obtain
\begin{equation}
\begin{split}
\left[
\begin{array}{c}
u_i-u_0 \\ 
v_i-v_0 \\
1, \\
d_i \\
\end{array}
\right]^T \left[
\begin{array}{cccc}
\frac{fb}{f_x} & 0 & 0 & 0\\ 
0 & \frac{fb}{f_y} & 0 & 0\\
0 & 0 & fb  &  0\\
0 & 0 & 0 & 1\\
\end{array}
\right]^T \mathbf{A} \\
\left[\begin{array}{cccc}
\frac{fb}{f_x} & 0 & 0 & 0\\ 
0 & \frac{fb}{f_y} & 0 & 0\\
0 & 0 & fb  &  0\\
0 & 0 & 0 & 1\\
\end{array}
\right]\left[
\begin{array}{c}
u_i-u_0 \\ 
v_i-v_0  \\
1 \\
d_i \\
\end{array}
\right] = 0.
\end{split}
\end{equation}
It is thus clear that a second order surface in $XYZ$ space has been transformed to another second order surface in $UVD$ space. The above proof could be extended to any order 3D surface directly.
\end{proof}

Therefore, we can conclude that any order parametric surfaces defined in the $XYZ$ space have a corresponding surface of the same order in the $UVD$ space. In other words, the transformation from $XYZ$ space to $UVD$ space is curve order preserving. 

%==============================================================================
\subsection{2D convolution VS 3D convolution}
State-of-the-art semantic labeling methods use deep convolutional network to learn the nonlinear mapping from image to dense semantic labeling, where the convolution is conducted in a 2D manner. As the neighboring relation is defined on the image plane, the 2D convolution may fail to extract feature with 3D geometry aware. Instead, 3D convolution in the 3D voxel space could integrate the appearance cues in 3D geometry aware manner, \ie, the 3D distance has been catered in convolution.   

Given a color image, the 2D convolution is expressed as Eq \ref{eq:2d}. The value of an unit at position $(u,v)$ in the $i^{th}$ feature map is denoted as $p_{i}^{u,v}$,
\begin{equation}
p_{i}^{uv}= \sum_k\sum_{m=0}^{M_i-1}\sum_{n=0}^{N_i-1}w_{ik}^{mn}p_{(i-1)k}^{(u+m)(v+n)},
\label{eq:2d}
\end{equation}
where $w_{ik}^{mn}$ is the coefficient at the position $(m,n)$ of the kernel connected to the $k^{th}$ feature map. $M,N$ are the height and width of the kernel. When the convolution is conducted in 3D, the value of an unit at position $(u,v,d)$ in the $i^{th}$ feature map denoted as $p_{i}^{u,v,d}$ is given by
\begin{equation}
p_{i}^{uvd}= \sum_k\sum_{m=0}^{M_i-1}\sum_{n=0}^{N_i-1}\sum_{l=0}^{L_i-1}w_{ik}^{mnl}p_{(i-1)k}^{(u+m)(v+n)(d+l)},
\label{eq:3d}
\end{equation}
where $d$ is the third dimension of the feature map and $L_i$ is the size of the 3D kernel along the third dimension. 3D convolution can extract features from both spatial and disparity dimensions. 

In Fig.~\ref{fig:2D_3D_convolution}, we compare 2D convolution and 3D convolution for an outdoor street scene. As observed from the illustration, the 2D convolution extracts feature in neighborhood defined on the image plane, which could involve points far away in 3D space. By contrast, 3D convolution in the $UVD$ space succeeds to extract features in a 3D geometry aware manner. 

\begin{figure}[!htp]
\begin{center} 
\subfigure{
\includegraphics[width=0.45\columnwidth]{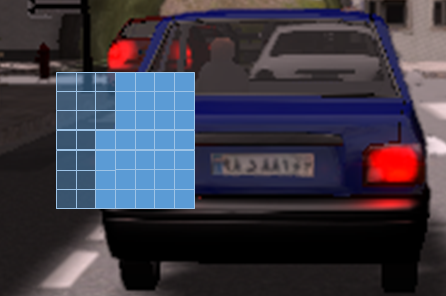} } \hspace{-.3cm}
\subfigure{
\includegraphics[width=0.45\columnwidth]{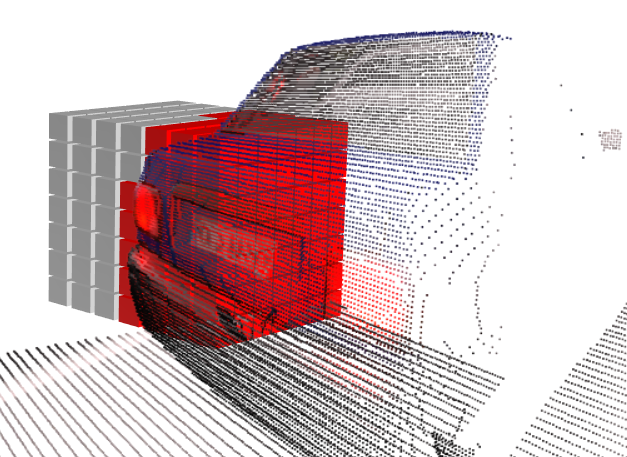} }
\end{center}
\caption{\label{fig:2D_3D_convolution}\small Illustration of 2D convolution and 3D convolution for semantic labeling. The left image demonstrates the widely used 2D convolution in extracting feature maps while the right image illustrates the corresponding 3D convolution conducted in the 3D voxel space. Note that the natural neighborhood relation is not preserved in the projection from 3D to 2D.}
\end{figure}

% \begin{figure}[!htp]
% \begin{center} 
%  \subfigure{
% \includegraphics[width=0.50\columnwidth]{3dPC.png} }
% \subfigure{
% \includegraphics[width=0.42\columnwidth]{2_5dPC.png} }
% \caption{\label{fig:3dpres} }
%  \end{center}
% \end{figure}

%===============================================================================================

\subsection{Network Architecture}
\begin{figure*}[!htp]
\centering
\includegraphics[width=2\columnwidth]{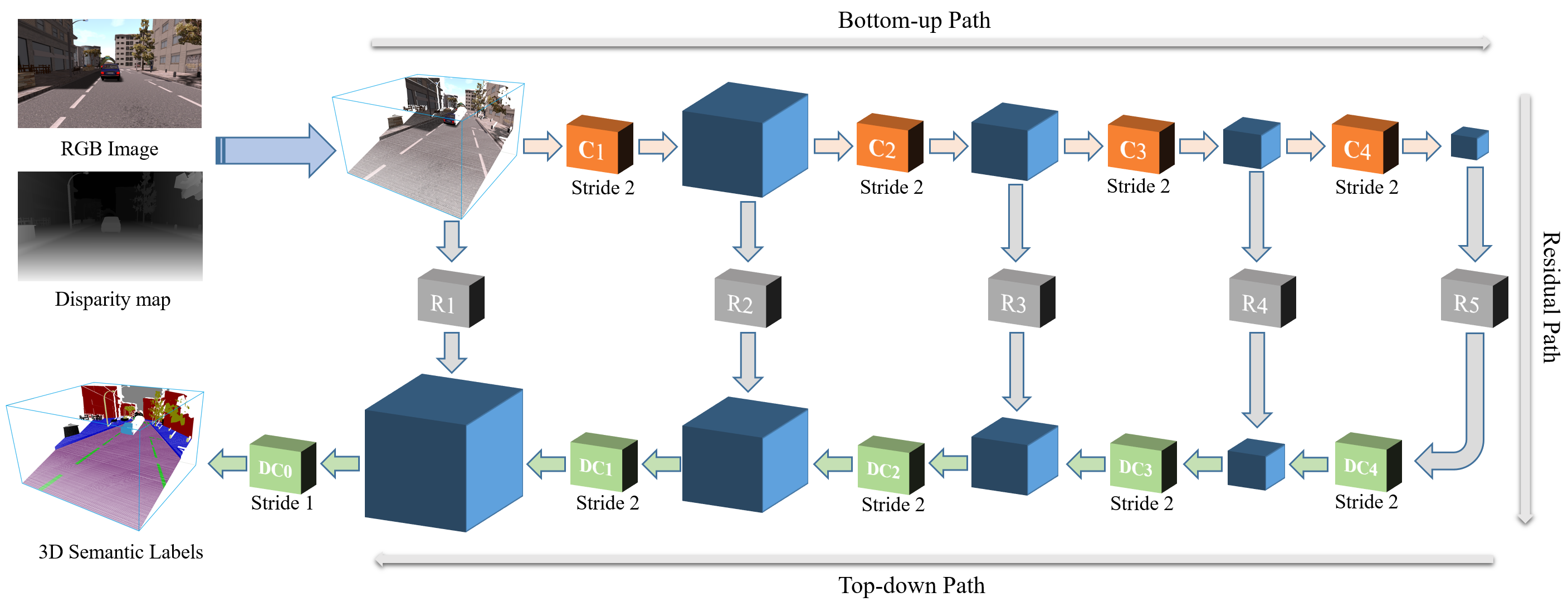} 
\caption{\label{fig:net} \small A nutshell of our 3D geometry-aware semantic labeling framework. The input RGB-D images are converted to the 3D voxel representation. $C_i$ denotes the 3D convolution layer that encodes geometric and contextual information, $R_i$ is residual module that connects low level features to the top-down pathway. $DC_i$ is the 3D deconvolution layer to decode geometric and contextual information. We achieve a 3D semantic labeling, which could be projected to 2D for the sake of comparison.}
\end{figure*}

Our goal is to assign each 3D point with a class label. A natural solution is to do 3D convolution on these point clouds. Also, since small objects such as traffic lights and signs play equally important role in semantic labeling, we adopt the idea of Top-Down Modulation (TDM) \cite{TDM2016}. We not only convert it to 3D, but also modify it to better suit for our case. Note that in our 3D representation, most voxelized 3D labels are 0s. For these points, identity mappings are optimal. Therefore we swap the lateral connection between Bottom-up features and Top-down features with residual connection, and let the solvers simply drive the weights of the multiple nonlinear layers toward zero to approach identity mappings. Formally, we define a block from Top-down path:
\begin{equation}
\mathbf{y} = \mathcal{F}(\mathbf{x}_{C_i},\{\mathbf{W}_i\})  + \mathbf{x}_{DC_i},
\label{eq:res}
\end{equation}
where $\mathbf{y}$ denotes the output of residual connection, $\mathbf{x}_{C_i}$ is the output from the $i$-th convolution layer and $\mathbf{x}_{DC_i}$ is output from the $i$-th deconvolution layer. The function $\mathcal{F}(\mathbf{x},\{\mathbf{W}_i\})$ is the residual mapping to learn. $+$ represents element-wise addition. Thus the dimensions of $\mathbf{x}_{R_i}$ and $\mathbf{x}_{DC_i}$ must be equal as in Eq. \ref{eq:res}.

In Fig.~\ref{fig:net}, we present a nutshell of our overall architecture of the proposed network. Given a color image and its corresponding disparity map, we first represent the RGB-D in the 3D $UVD$ space as defined in Section 3.1. In the Bottom-up phase, the 3D volume ($H\times W\times (D+1)\times Ch$) passes through a series of 3D convolutional layers ($C_i$) with the same kernel size $3\times 3 \times 3$ and a stride 2 until achieving an encoded feature volume with dimension $(1/16) H\times (1/16)W\times (1/16)(D+1)\times F$, where $H,W,D,Ch,F$ represent the height, width, disparity levels, and number of channels and features respectively. In the Top-down phase, a mirrored process scales up the encoded feature volume back to the original size by swapping the 3D convolution with 3D deconvolution. For each scale, we apply our Res-TDM with a residual module $R_i$. Each $R_i$ consists of two 3D convolution layers with the same kernel size $3\times 3 \times 3$ and stride 1. 

We employ the cross-entropy loss given by Eq. \ref{eq:loss} as our loss function for training the network. 
\begin{equation}
L(\mathbf{w}) = -\frac{1}{N+1}\sum_n[y_n\log \hat{y}_n + (1-y_n)\log (1-\hat{y}_n)]
\label{eq:loss}
\end{equation}
%where $\hat{y}_n = g(\mathbf{w}\dot $, 
where $\hat{y}_n = g(\mathbf{w}\dot{\mathbf{x}}_n)$ with logistic function $g(z)$. $\mathbf{w}$ is the vector of weights and each sample is labeled by $n = 0,1,2,...,N-1$. Note that there is a large variation in the number of pixels for each class. Despite of the imbalance distribution of valid labels, we only have $1/D$ valid labels in total. In other words, if the network predict all zeros, it can still achieve a training accuracy higher than $95\%$. In order to avoid our network drop to this local minimum, we apply a residual module $R$ directly from the input that impose the network only to learn parameters around the areas with non-zero input.

Training our end-to-end pixel-wise semantic labeling network is very straightforward, which can be trained under the supervision of ground-truth semantic labels. Supervision is applied on the volumized 3D predictions and labels. All void 3D points (e.g., points before an object or behind an object) cannot be ignored and should be also given a void label 0. In the prediction phase, we perform max pooling along the disparity dimension to convert the 3D volume back to a 2D image and calculate the errors. This strategy can crease the robustness when dealing with noises and errors on disparity maps. For example, there is no guarantee that our network will predict the right label at the exact disparity level. When the disparity map is noisy, points with the same labels may have very different disparity levels. In this case, the network may predict the right label on the similar disparity level rather than the noise one.

%====================================================
\section{Evaluation}
In this section, we evaluate our proposed method with a comparison to alternative approaches and present an ablation study to better understand the proposed framework. Our method is evaluated on both synthetic and real datasets.

%===================================================
\subsection{Dataset}
For synthetic data, we use the SYNTHetic Collection of Imagery and Annotations (SYNTHIA) dataset \cite{RosCVPR16}, which contains 3 subsets: synthia-rand-cvpr16, synthia-rand-cityscapes and synthia-video-sequence. We choose the synthia-rand-cityscapes subset for experiments. It consists of 23 classes and a total of $9400$ frames of outdoor scene with different weather and lighting conditions as well as randomly generated viewing angles. Since the dataset does not provide training and testing split, we randomly select 6000 frames for training, 1900 frames for validation and the remaining 1500 frames for testing. We manually convert the given ground truth depth maps to scaled inverse depth map in the range of $[0,191]$ and resize the input image to $80\times 128$.

%Still two disparity????
For real data experiment, we use the Cityscape dataset \cite{Cordts2016Cityscapes}, which contains 5,000 stereo frames of fine annotated ground truth semantic labels. We choose 2975 frames for training and use 500 frames for testing. In experiments, we compute the disparity map by using state-of-the-art stereo matching method, which is truncated to the range $[0,111]$. The input images are resized to the resolution of $256\times 512$. 

\textbf{Data augmentation} We employ the mirror manipulation to augment the training examples for both datasets, since it maintains the geometry relationships.

%=================================================
\subsection{Optimization}
The proposed network architecture was implemented with Tensorflow \cite{tensorflow}. We employed the RMSProp \cite{Tieleman2012} with a constant learning rate of $1\times 10^{-3}$ to optimize all models in end-to-end manner. For the ``Synthia'' dataset, we normalized input images' RGB values to $[-1,1]$ and trained our network from a random initialization for 50 epochs, which took 50 hours to converge by using a single NVIDIA Pascal Titan-X GPU and 1 second per frame in testing phase. However, the testing global accuracy climbs up to over $80\%$ within one epoch. For the Cityscape dataset, we trained our network (S3D) with color input for 30 epochs. In order to fit the 12G memory, we reduce the number of disparity levels to 48. We also trained our S3D network with feature input. The features were extracted from Resnet-38\cite{word.zifeng.2016} with the same input dimension. The input feature dimension of our network is $32\times 64 \times 512$. We added 3 extra upsampling layers in order to match the output resolution. The network converged quickly within 14 epochs. Note we did not use any post-processing to refine the results.

%=================================================
\subsection{Evaluation metric}
We measure the semantic labeling performance of our network with three metrics. Denote the total number of classes as $k$, $p_{ij}$ as the amount of pixels belonging to class $i$ which are predicted to be class $j$, the Global accuracy $G = \frac{\sum_i p_{ii}}{\sum_i\sum_j p_{ij}}$ measures the percentage of pixels correctly classified in the dataset. The Class Average Accuracy $C = \frac{1}{k+1}\sum_i\frac{p_{ii}}{\sum_j p_{ij}} $ normalizes the accuracy over the classes, therefore all classes share the same weight under this metric. Mean intersection over union $mIoU = \frac{1}{k+1}\sum_i\frac{p_{ii}}{\sum_j p_{ij}+\sum_j p_{ji}-p_{ij}} $ is used in the Cityscapes benchmark \cite{Cordts2016Cityscapes}. It is a more strict metric than class average accuracy since it penalizes false positive predictions.

% \begin{compactitem}
% \item Global Accuracy: $G = \frac{\sum_i p_{ii}}{\sum_i\sum_j p_{ij}} $
% \item Class Average Accuracy: $C = \frac{1}{k+1}\sum_i\frac{p_{ii}}{\sum_j p_{ij}} $
% \item Mean Intersection over Union: $mIoU = \frac{1}{k+1}\sum_i\frac{p_{ii}}{\sum_j p_{ij}+\sum_j p_{ji}-p_{ij}} $
% \end{compactitem}

%==================================================
\subsection{Experimental results}

\begin{figure*}[htb]
\begin{center} 
% \vspace{-.4cm}
\subfigure{
\includegraphics[width=0.31\columnwidth]{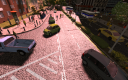} } \hspace{-.3cm}
\subfigure{
\includegraphics[width=0.31\columnwidth]{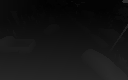} } \hspace{-.3cm}
\subfigure{
\includegraphics[width=0.31\columnwidth]{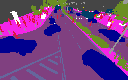} } \hspace{-.3cm}
\subfigure{
\includegraphics[width=0.31\columnwidth]{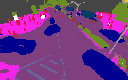} } \hspace{-.3cm}
\subfigure{
\includegraphics[width=0.31\columnwidth]{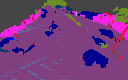} } \hspace{-.3cm}
\subfigure{
\includegraphics[width=0.31\columnwidth]{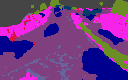} }

 \vspace{-.2cm}
\subfigure{
\includegraphics[width=0.31\columnwidth]{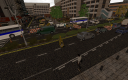} } \hspace{-.3cm}
\subfigure{
\includegraphics[width=0.31\columnwidth]{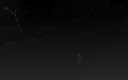} } \hspace{-.3cm}
\subfigure{
\includegraphics[width=0.31\columnwidth]{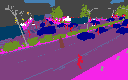} } \hspace{-.3cm}
\subfigure{
\includegraphics[width=0.31\columnwidth]{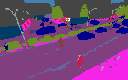} } \hspace{-.3cm}
\subfigure{
\includegraphics[width=0.31\columnwidth]{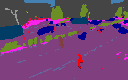} } \hspace{-.3cm}
\subfigure{
\includegraphics[width=0.31\columnwidth]{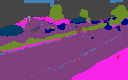} } 

 \vspace{-.2cm}
\subfigure{
\includegraphics[width=0.31\columnwidth]{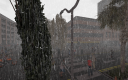} } \hspace{-.3cm}
\subfigure{
\includegraphics[width=0.31\columnwidth]{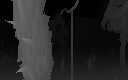} } \hspace{-.3cm}
\subfigure{
\includegraphics[width=0.31\columnwidth]{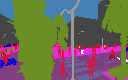} } \hspace{-.3cm}
\subfigure{
\includegraphics[width=0.31\columnwidth]{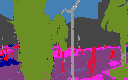} } \hspace{-.3cm}
\subfigure{
\includegraphics[width=0.31\columnwidth]{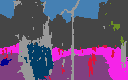} } \hspace{-.3cm}
\subfigure{
\includegraphics[width=0.31\columnwidth]{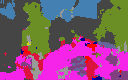} } 

\vspace{-.2cm}
\setcounter{subfigure}{0}
\subfigure[\scriptsize Color image]{
\includegraphics[width=0.31\columnwidth]{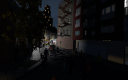} } \hspace{-.3cm}
\subfigure[\scriptsize Disparity]{
\includegraphics[width=0.31\columnwidth]{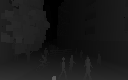} } \hspace{-.3cm}
\subfigure[\scriptsize Ground truth]{
\includegraphics[width=0.31\columnwidth]{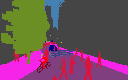} } \hspace{-.3cm}
\subfigure[\scriptsize Our method]{
\includegraphics[width=0.31\columnwidth]{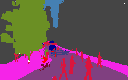} } \hspace{-.3cm}
\subfigure[\scriptsize FuseNet]{
\includegraphics[width=0.31\columnwidth]{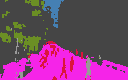} } \hspace{-.3cm}
\subfigure[\scriptsize SegNet]{
\includegraphics[width=0.31\columnwidth]{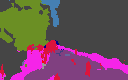} } 
\caption{\label{fig:sy1} \small \textbf{Quality comparison on the SYNTHIA dataset} We select images with different lighting and weather conditions as well as different viewing angles. Our method (d) shows superior performance, particularly it generates sharp boundaries for small objects. FuseNet (e) and SegNet (f) achieve similar performance but with the help of disparity map, FuseNet (e) captures more small objects such as pedestrians and poles.}
 \end{center}
\end{figure*}

\begin{table*}[h]
%\tiny
\centering
\caption{\label{tab:sy}\small Performance evaluation on the SYNTHIA dataset}
\label{synthia_result}
%\resizebox{\textwidth}{!}{
\tabcolsep=0.145cm
\begin{tabularx}{\textwidth}{c|cccccccccccccccccc|ccc}
     Method                              &\rot{sky}&\rot{Building}&\rot{Road}&\rot{Sidewalk}&\rot{Fence}&\rot{Vegetation}&\rot{Pole}&\rot{Car}&\rot{Traffic sign}&\rot{Pedestrian}&\rot{Bicycle}&\rot{Motorcycle}&\rot{Road-work}&\rot{Traffic light} &\rot{Rider}   &\rot{Bus}   &\rot{Wall} &\rot{Lanemarking}& \rot{Class avg.} &\rot{Global avg.}&\rot{mIoU} \\ \hline
SegNet\cite{badrinarayanan2015segnet}  &95.5 &93.3 &85.0 &87.2 &24.9 &79.9   &16.9 &60.8 &0.2 &50.2 &1.4
 &10.6 &40.3 &0.0 &11.2 &65.5  &18.6 &\textbf{45.9} &43.7 &82.6 &36.7                  \\ \hline
FuseNet\cite{fusenet2016accv} &92.4 &94.5 &79.9 &70.2 &35.6 &73.0 &29.9  &64.4 &2.5 &57.5 &2.8 &9.4 &46.4 &2.6 &16.4 &60.9 &13.7 &20.9 &42.9 &78.1 &35.9             \\ \hline
%S2D(RGB-d)  &98.0 &92.0 &62.6 &82.2 &41.6 &80.9 &33.4 &68.6 &2.5  &66.9  &0.3 &8.9   &60.2 &0.9  &1.6  &37.0  &0.0 &16.6 &41.9 &77.9 &33.6    \\ \hline
%S3D(d only)     &100.0 &98.2 &91.4 &91.7 &59.5 &92.6 &52.5 &83.8 &0.3 &80.8 &12.6 &2.8 &46.4 &0.7 &25.9 &76.0 &38.0 &0.1 &52.9 &89.9 &47.0                         \\ \hline
%S3D(SegNet-d) &94.9 &96.1 &89.9&\textbf{92.0}&51.5&90.1 &35.2&81.5&5.6&66.2&06.1&26.5&57.2&10.9&28.7&79.1&49.2&42.9 &55.7 &88.6 &48.1\\ \hline
S3D(ours)     &\textbf{97.4} &\textbf{97.1} &\textbf{91.8} &\textbf{91.2} &\textbf{59.6} &\textbf{90.7} &\textbf{47.8} &\textbf{87.6} &\textbf{15.5} &\textbf{72.9} &\textbf{13.5} &\textbf{36.4} &\textbf{72.24} &\textbf{31.7} &\textbf{32.6} &\textbf{83.3} &\textbf{58.2} &42.1 &\textbf{62.3} &\textbf{90.2} &\textbf{54.5}                         \\ \hline
\end{tabularx}
%}
\end{table*}

\textbf{Results on Synthia.}
% \begin{table}[!htp]
% \centering
% \caption{Evaluation results on synthia dataset}
% \label{campare1}
% \begin{tabular}{|c|c|c|c|}
% \hline
%                                           & Global Accuracy & Class Average Accuracy & mean IoU \\ \hline
% SegNet\cite{badrinarayanan2015segnet}     & $82.27$         & $43.67$                & $36.41$ \\ \hline
% FuseNet\cite{fusenet2016accv}             & $76.48$         & $41.76$                & $33.83$  \\ \hline
% S2D (RGB)                                 & $80.71$         & $46.23$                & $36.12$  \\ \hline
% S2D (RGB-D)                               & $77.63$         & $40.20$                & $33.12$  \\ \hline
% S3D                                       & $90.02$         & $63.11$                & $54.52$  \\ \hline
% FuseNet noise\cite{fusenet2016accv}       & $76.22$         & $39.93$                & $32.50$  \\ \hline
% S3D noise                                 & $87.92$         & $54.32$                & $47.78$  \\ \hline
% \end{tabular}
% \end{table}
In Table \ref{tab:sy}, we quantitatively compare our method (S3D) with state-of-the-art RGB semantic segmentation approach ``SegNet'' \cite{badrinarayanan2015segnet} and RGB-D based approach ``FuseNet'' \cite{fusenet2016accv}. For SegNet and FuseNet, we use the same input size and initialize the network parameters from the VGG model pre-trained on ImageNet. We train SegNet for 790 epochs and 230 epochs for FuseNet. For FuseNet, we use the same scaled inverse depth maps to train our network. Our method significantly outperforms competing methods with a notable margin under all three metrics: \textbf{18.6\%} on class average accuracy, \textbf{7.6\%} on global accuracy and \textbf{17.8\%} on mIoU. 
%To better illustrate the efficiency of our model, we plug trained features from SegNet into our network. We shall see that all metrics have been significantly improved with \textbf{12.0\%}, \textbf{6.0\%}, \textbf{11.4\%} on class average accuracy, global accuracy and mIoU respectively. 
Note that there are 4 classes never show up in testing set, so we remove them from the table and during the error calculation. In Fig.~\ref{fig:sy1}, we present qualitative comparison between our method and state-of-the-art methods on the Synthia dataset, which clearly demonstrates the superior performance of our method.

% \begin{table}[!htp]
% \centering
% \caption{Evaluation results on SYNTHIA-VIDEO-SEQUENCES}
% \label{campare1}
% \begin{tabular}{|c|c|c|c|}
% \hline
%                                           & Global Accuracy & Class Average Accuracy & mean IoU \\ \hline
% SegNet\cite{badrinarayanan2015segnet}     & $92.15$         & $57.58$                & $48.62$ \\ \hline
% ResNet-38\cite{word.zifeng.2016}          & $82.29$         & $50.67$                & $40.38$ \\ \hline
% FuseNet\cite{fusenet2016accv}             & $87.32$         & $52.21$                & $42.09$  \\ \hline
% S2D   (RGB)                               & $90.97$         & $60.67$                & $50.13$  \\ \hline
% S2D   (RGB-D)                             & $93.77$         & $68.42$                & $59.60$  \\ \hline
% S3D   (RGB-D)                             & $93.96$         & $68.85$                & $58.64$  \\ \hline
% S3D   (Depth only)                        & $85.24$         & $54.22$                & $43.78$  \\ \hline
% \end{tabular}
% \end{table}

\textbf{Results on Cityscapes.} Quantitative comparison with state-of-the-art semantic labeling methods on the Cityscapes dataset is shown in Table \ref{tab:cs}. The weight of FuseNet and SegNet are initialized from the VGG model trained on ImageNet. We also compare our method with the top performing one on the Cityscapes benchmark: ResNet-38\cite{word.zifeng.2016}. Given RGB-D pair as input, we achieve similar performance with FuseNet. However, by swapping RGB image with trained features, our method outperforms all competing methods with a margin \textbf{1.2\%}, \textbf{0.6\%}, \textbf{1.0\%} for class average accuracy, global accuracy and mIoU respectively.  The margin is not as clear as previous one is due to the noise and errors in the disparity map. However, our algorithm still successfully squeezed useful information from it and increased the performance. Advanced disparity recovery algorithm \cite{SsSMnet} should lead to better performance.
In Fig.~\ref{fig:cs3}, we present qualitative comparison between our framework and state-of-the-art methods on the Cityscapes dataset, which proves the superiority of our method.

\begin{figure*}[htb]
\begin{center} 
% \vspace{-.4cm}
\subfigure{
\includegraphics[width=0.31\columnwidth]{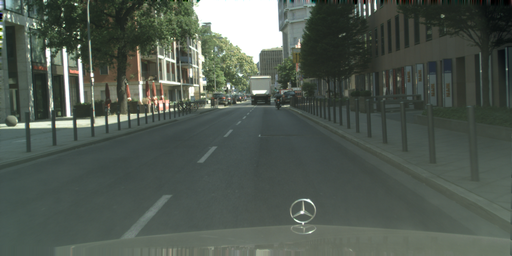} } \hspace{-.3cm}
\subfigure{
\includegraphics[width=0.31\columnwidth]{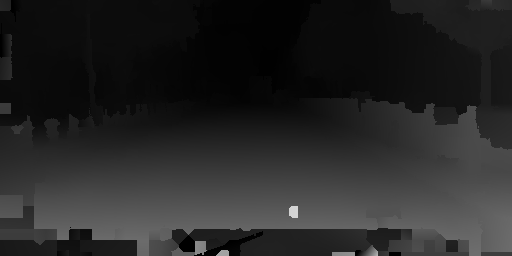} } \hspace{-.3cm}
\subfigure{
\includegraphics[width=0.31\columnwidth]{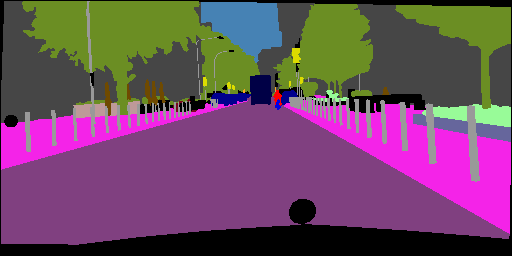} } \hspace{-.3cm}
\subfigure{
\includegraphics[width=0.31\columnwidth]{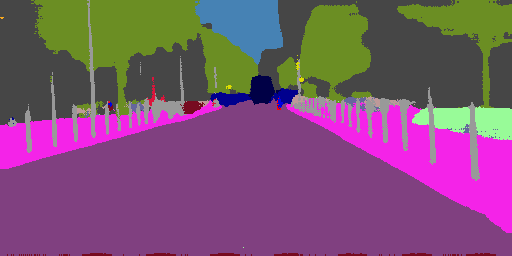} } \hspace{-.3cm}
\subfigure{
\includegraphics[width=0.31\columnwidth]{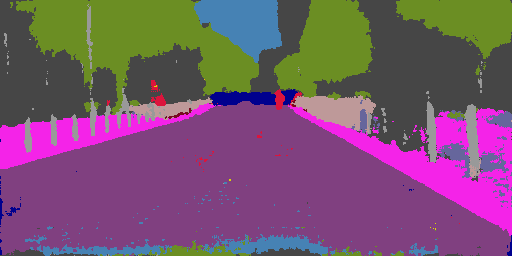} } \hspace{-.3cm}
\subfigure{
\includegraphics[width=0.31\columnwidth]{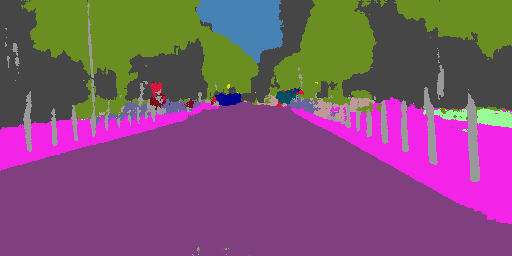} }

 \vspace{-.2cm}
\subfigure{
\includegraphics[width=0.31\columnwidth]{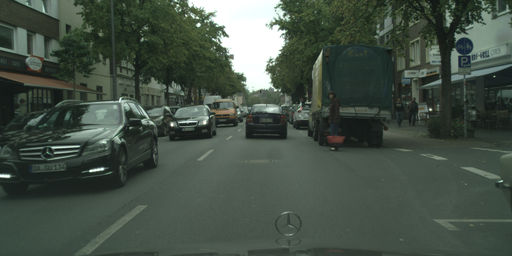} } \hspace{-.3cm}
\subfigure{
\includegraphics[width=0.31\columnwidth]{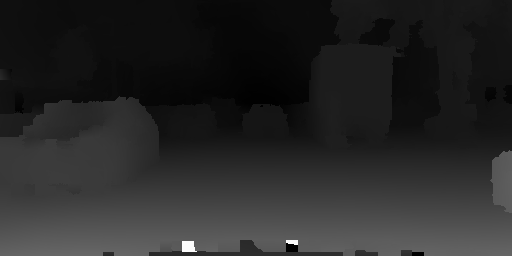} } \hspace{-.3cm}
\subfigure{
\includegraphics[width=0.31\columnwidth]{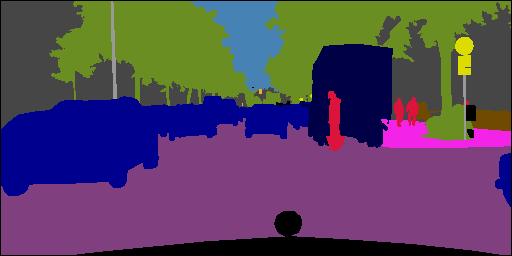} } \hspace{-.3cm}
\subfigure{
\includegraphics[width=0.31\columnwidth]{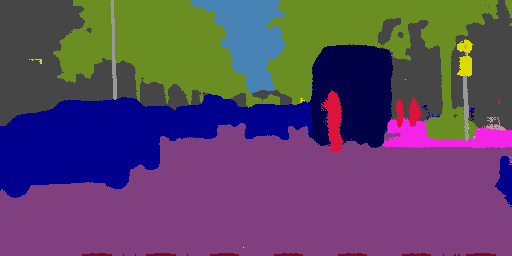} } \hspace{-.3cm}
\subfigure{
\includegraphics[width=0.31\columnwidth]{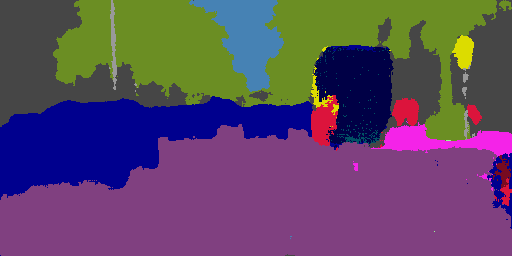} } \hspace{-.3cm}
\subfigure{
\includegraphics[width=0.31\columnwidth]{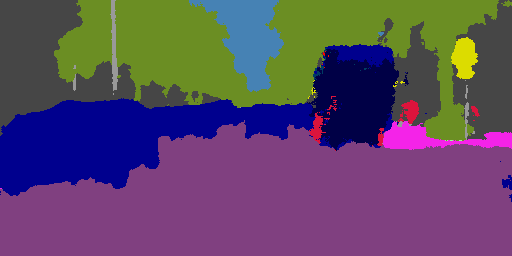} }  

 \vspace{-.2cm}
\subfigure{
\includegraphics[width=0.31\columnwidth]{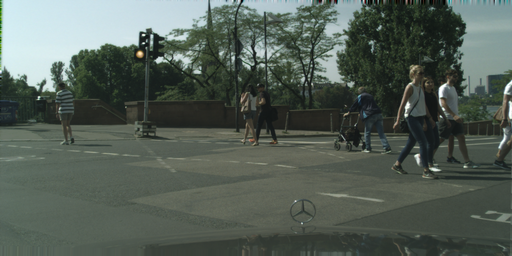} } \hspace{-.3cm}
\subfigure{
\includegraphics[width=0.31\columnwidth]{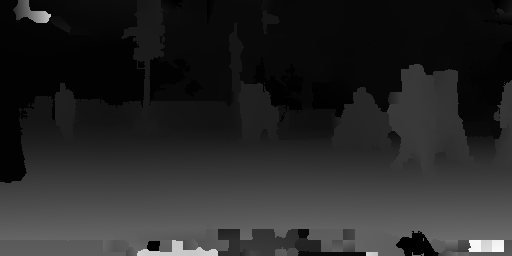} } \hspace{-.3cm}
\subfigure{
\includegraphics[width=0.31\columnwidth]{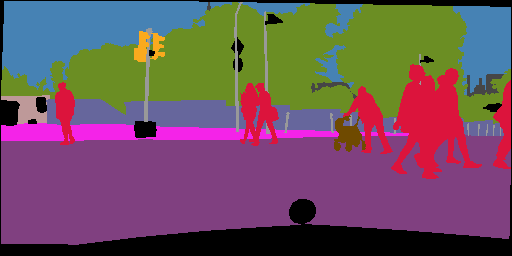} } \hspace{-.3cm}
\subfigure{
\includegraphics[width=0.31\columnwidth]{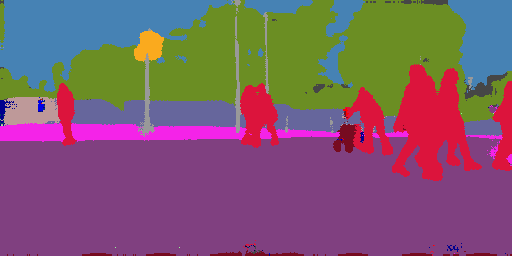} } \hspace{-.3cm}
\subfigure{
\includegraphics[width=0.31\columnwidth]{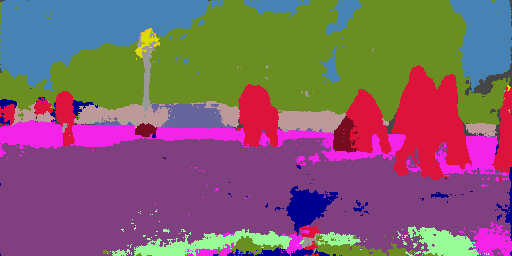} } \hspace{-.3cm}
\subfigure{
\includegraphics[width=0.31\columnwidth]{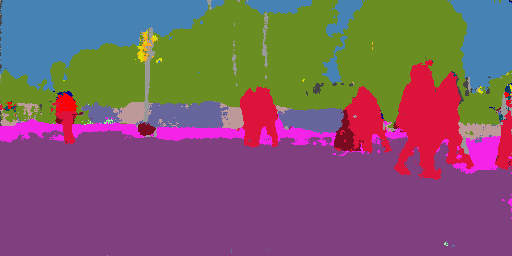} } 

\vspace{-.2cm}
\setcounter{subfigure}{0}
\subfigure[\scriptsize Color image]{
\includegraphics[width=0.31\columnwidth]{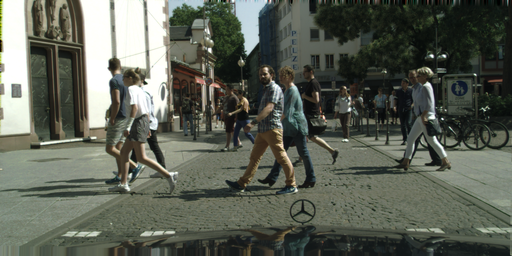} } \hspace{-.3cm}
\subfigure[\scriptsize Disparity]{
\includegraphics[width=0.31\columnwidth]{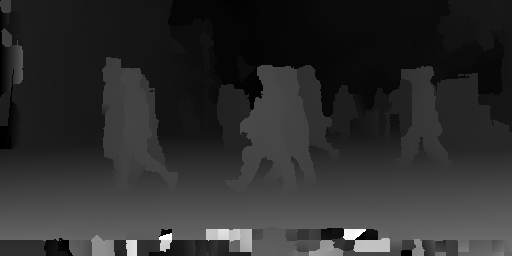} } \hspace{-.3cm}
\subfigure[\scriptsize Ground truth]{
\includegraphics[width=0.31\columnwidth]{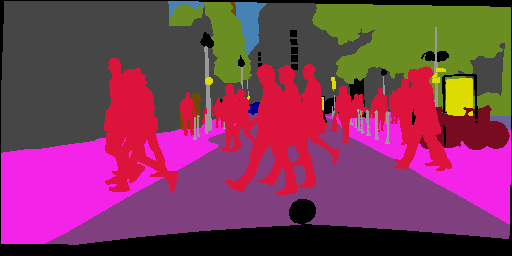} } \hspace{-.3cm}
\subfigure[\scriptsize Our method]{
\includegraphics[width=0.31\columnwidth]{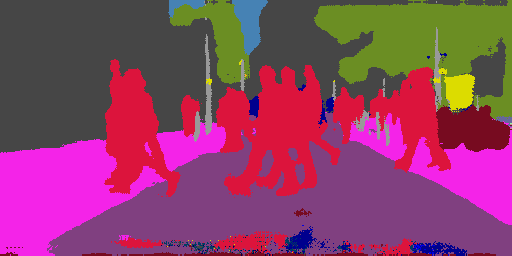} } \hspace{-.3cm}
\subfigure[\scriptsize FuseNet]{
\includegraphics[width=0.31\columnwidth]{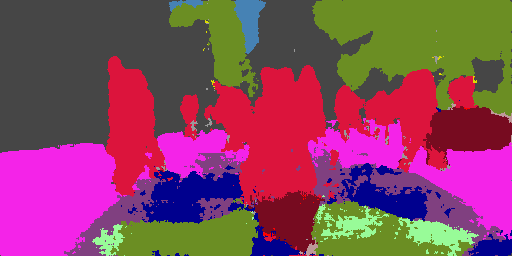} } \hspace{-.3cm}
\subfigure[\scriptsize SegNet]{
\includegraphics[width=0.31\columnwidth]{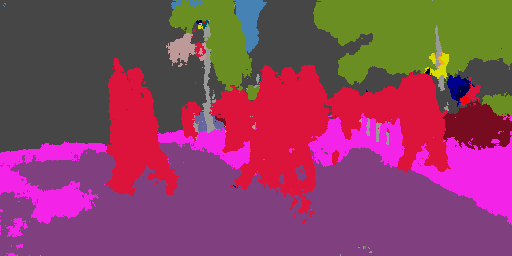} } 
\caption{\label{fig:cs3}\small \textbf{Qualitative evaluation on the Cityscapes dataset.} Our method with feature and disparity inputs (d) clearly outperforms the competing methods. The shape of pedestrians and poles are well preserved in our predictions. We shall see that FuseNet is effected by the noisy disparity map that has worse performance than SegNet. Note the invalid label is colored with black.}
 \end{center}
\end{figure*}

\begin{table*}[!htp]
%\tiny
\centering
\caption{\label{tab:cs}\small Performance evaluation on Cityscapes validation set}
\label{Cityscapes_result}
%\resizebox{\textwidth}{!}{
\tabcolsep=0.12cm
\begin{tabularx}{\textwidth}{c|ccccccccccccccccccc|ccc}
     Method                              &\rot{Road}&\rot{Sidewalk}&\rot{Building}&\rot{Wall}&\rot{Fence}&\rot{Pole}&\rot{Traffic light}&\rot{Traffic sign}&\rot{Vegetation}&\rot{Terrain}&\rot{Sky}&\rot{Person}&\rot{Rider}&\rot{Car}&\rot{Truck}&\rot{Bus} &\rot{Train}  & \rot{Motorcycle} &\rot{Bicycle}  & \rot{Class avg.} &\rot{Global avg.}&\rot{mIoU} \\ \hline
SegNet\cite{badrinarayanan2015segnet}  &97.4 &82.0 &92.5 &35.2 &35.8 &40.9  &8.2 &40.5 &93.2 &62.6 &96.2 &71.9
 &11.8 &92.7 &46.3 &41.6 &27.2  &9.9 &61.2 &55.1 &90.5 &45.8                  \\ \hline
ResNet-38\cite{word.zifeng.2016}         &97.9&81.1 &\textbf{93.6} &62.0 &58.4 &41.9 &55.4 &\textbf{62.5} &94.2 &63.8 &92.3 &\textbf{79.1} &\textbf{54.9} &\textbf{94.7} &74.5 &76.5 &66.1&\textbf{61.4}&\textbf{75.1} & 72.9 & 92.2 &63.3 \\  \hline
FuseNet\cite{fusenet2016accv} &88.6 &82.1 &93.3 &25.4 &48.2 &42.7 &0.5 &47.0 &94.5 & 38.5 &\textbf{96.8} &75.8 &0.5 &93.6 &56.9 &2.2 &12.4 &0.0 &60.2 &50.5 &87.4 &39.1             \\ \hline
%S2D(RGB-d)     &88.6 &69.3 &90.0 &6.3 &11.9 &33.7 &34.5 &39.0 &91.6 &52.6 &95.3 &62.7 &0.9 &91.3 &34.8 &22.2 &9.8 &0.6 &20.0 &45.5 &84.5 &36.3                         \\ \hline
S3D(RGB-d)     &94.5 &72.2 &86.1 &15.7 &17.7 &34.0 &38.1 &52.3 &91.1 &65.9 &96.2 &64.3 &8.0 &86.9 &22.0 &20.5 &14.6 &3.5 &28.2 &48.2 &86.9 &39.1                         \\ \hline
S3D(feature-d) &\textbf{98.0} &\textbf{87.4} &93.3 &\textbf{66.2} &\textbf{71.3} &\textbf{46.8} &\textbf{60.2} &62.3 &\textbf{95.2} &\textbf{67.4} &92.8 &78.0 &41.8 &91.8 &\textbf{78.7} &\textbf{81.2} &\textbf{73.3} &47.5 &75.0  &\textbf{74.1} &\textbf{92.8} & \textbf{64.3}                           \\ \hline
\end{tabularx}
%}
\end{table*}

\textbf{Ablation study}
To better understand the effectiveness of our 3D voxel representation, we perform ablation analysis and present the results in Table \ref{tab:ab}. S2D is the 2D version of our algorithm that replaces all 3D convolutions with 2D ones, where we stack the RGB image and disparity map into 4 channel input and plug into the S2D. S3D (Depth only) is the one with colorless ``point clouds'' which only provides shape information. According to this study, 3D voxel representation significantly improves the performance by $20.9\%$ in mIoU.

\begin{table*}[!htp]
%\tiny
\centering
\caption{Ablation study on the SYNTHIA dataset}
\label{tab:ab}
%\resizebox{\textwidth}{!}{
\tabcolsep=0.135cm
\begin{tabularx}{\textwidth}{c|cccccccccccccccccc|ccc}
     Method                              &\rot{sky}&\rot{Building}&\rot{Road}&\rot{Sidewalk}&\rot{Fence}&\rot{Vegetation}&\rot{Pole}&\rot{Car}&\rot{Traffic sign}&\rot{Pedestrian}&\rot{Bicycle}&\rot{Motorcycle}&\rot{Road-work}&\rot{Traffic light} &\rot{Rider}   &\rot{Bus}   &\rot{Wall} &\rot{Lanemarking}& \rot{Class avg.} &\rot{Global avg.}&\rot{mIoU} \\ \hline
S2D(RGB-d)  &98.0 &92.0 &62.6 &82.2 &41.6 &80.9 &33.4 &68.6 &2.5  &66.9  &0.3 &8.9   &60.2 &0.9  &1.6  &37.0  &0.0 &16.6 &41.9 &77.9 &33.6    \\ \hline
S3D(d only)     &\textbf{100.0} &\textbf{98.2} &91.4 &\textbf{91.7} &59.5 &\textbf{92.6} &\textbf{52.5} &83.8 &0.3 &\textbf{80.8} &12.6 &2.8 &46.4 &0.7 &25.9 &76.0 &38.0 &0.1 &52.9 &89.9 &47.0                         \\ \hline
S3D(RGB-d)     &97.4 &97.1 &\textbf{91.8} &91.2 &\textbf{59.6} &90.7 &47.8 &\textbf{87.6} &\textbf{15.5} &72.9 &\textbf{13.5} &\textbf{36.4} &\textbf{72.24} &\textbf{31.7} &\textbf{32.6} &\textbf{83.3} &\textbf{58.2} &\textbf{42.1} &\textbf{62.3} &\textbf{90.2} &\textbf{54.5}                         \\ \hline
\end{tabularx}
%}
\end{table*}

\section{Conclusion}
In this paper, we have proposed a 3D voxel representation to integrate both appearance and depth information and a corresponding light-weight 3D Res-TDM network architecture for 3D geometry aware semantic segmentation. Our method provides an efficient and effective way to use geometric information to achieve better semantic labeling. Experiments on the ``Synthia'' and ``Cityscape'' datasets demonstrate that direct 3D convolution with our light-weight Res-TDM network can lead to superior performance, suggesting that such a simple 3D representation with Res-TDM is effective in incorporating 3D geometric information.

\section*{Acknowledgment}
We gratefully acknowledge the support of NVIDIA Corporation with donation of TITAN Xp GPU used for this research, as well a NVIDIA Drive-PX2 platform for an autonomous driving project.  YZ's PhD scholarship is funded by CSIRO Data61. Y. Dai was supported in part by National 1000 Young Talents Plan of China, Natural Science Foundation of China (61420106007, 61671387), and ARC grant (DE140100180). H. Li’s work is funded in part by Australia ARC Centre of Excellence for Robotic Vision (CE140100016).
\balance
{\footnotesize
\bibliographystyle{unsrt}
\bibliography{semantic}
}

\end{document}